%% file: main.tex
\documentclass[10pt]{article} 

\usepackage[preprint]{rlj} 

\usepackage[utf8]{inputenc} 
\usepackage[T1]{fontenc}    
\usepackage{hyperref}       
\usepackage{url}            
\usepackage{booktabs}       
\usepackage{amsfonts}       
\usepackage{amsmath}
\usepackage{amsthm}         
\usepackage{nicefrac}       
\usepackage{microtype}      
\usepackage{xcolor}         
\usepackage{algorithm}
\usepackage[noend]{algpseudocode}
\newtheorem{proposition}{Proposition}

\usepackage{amssymb}            
\usepackage{mathtools}          
\usepackage{mathrsfs}           
\usepackage{graphicx}           
\usepackage{subcaption}         
\usepackage[space]{grffile}     
\usepackage{url}                

\title{Data-Augmented Game Starts for Accelerating \\Self-Play Exploration in Imperfect Information Games
}

\setrunningtitle{Data-Augmented Game Starts for Accelerating Self-Play Exploration in Imperfect Information Games}

\author{JB Lanier\textsuperscript{1, $*$}, 
Nathan Monette\textsuperscript{2, $*$}, Pierre Baldi\textsuperscript{1}, Roy Fox\textsuperscript{1}}

\emails{jblanier@uci.edu, nathan.monette@cs.ox.ac.uk}

\affiliations{
$^{1}$\textbf{University of California Irvine}\\
$^{2}$\textbf{University of Oxford}
\par 
$^*$ Equal contribution
}

\contribution{
    We provide a new benchmark for 2p0s games with long horizons and exploration challenges with fast analytic exploitability calculations.
    }
    {
    Previous work in imperfect information games \citep{rudolph2025reevaluatingpolicygradientmethods, sokota2023a} use environments that are easy to compute optimality, but which have little or no elements of exploration with sparse rewards.
    }

\contribution{
    We show that DAGS can enable regularized policy gradient algorithms to reach low exploitability in games with significantly larger exploration challenges under a fixed computational budget.
    }
    {
    To achieve good performance, we use a baseline solver that regularizes toward an exponential moving average of the policy.
    }

\contribution{We demonstrate that training with DAGS can bias the learned equilibrium in imperfect-information games, and we provide a mitigation to this bias problem.
    }{Prior related work does not consider imperfect information games, the setting where this issue arises.}

\keywords{Multi-Agent Reinforcement Learning, Imperfect Information Games} 

\newcommand{\methodlong}{Data-Augmented Game Starts}
\newcommand{\methodshort}{DAGS}

\summary{
Finding approximate equilibria for large-scale imperfect-information competitive games such as StarCraft, Dota, and CounterStrike remains computationally infeasible due to sparse rewards and challenging exploration over long horizons. We propose \methodlong{} (\methodshort{}), a starting-state sampling strategy for two-player zero-sum (2p0s) games that initializes reinforcement learning episodes at intermediate states drawn from offline gameplay data. Our key assumption is that demonstrations from skilled humans provide good coverage of high-level strategies relevant to equilibrium play, allowing agents to explore strategically relevant subgames without first needing to learn policies that reach them.

We evaluate \methodshort{} using synthetic datasets and analytically tractable, long-horizon control variants of two-player Kuhn Poker, Goofspiel, and a counterexample game designed to penalize biased beliefs over hidden information. Under fixed computational budgets, \methodshort{} enables regularized policy gradient methods to solve games with significantly more challenging exploration. We show that augmenting starting state distributions when solving imperfect information games can lead to biased equilibria, and we provide a straightforward mitigation in the form of multi-task observation flags.
}

\begin{document}

\makeCover  
\maketitle  

\begin{abstract}
Finding approximate equilibria for large-scale imperfect-information competitive games such as StarCraft, Dota, and CounterStrike remains computationally infeasible due to sparse rewards and challenging exploration over long horizons. In this paper, we propose a multi-agent starting-state sampling strategy designed to substantially accelerate online exploration in regularized policy-gradient game methods for two-player zero-sum (2p0s) games. Motivated by an assumption that offline demonstrations from skilled humans can provide good coverage of high-level strategies relevant to equilibrium play, we propose the initialization of reinforcement learning data collection at intermediate states sampled from offline data to facilitate exploration of strategically relevant subgames. Referring to this method as \methodlong{} (\methodshort{}), we perform experiments using synthetic datasets and analytically tractable, long-horizon control variants of two-player Kuhn Poker, Goofspiel, and a counterexample game designed to penalize biased beliefs over hidden information.
Under fixed computational budgets, \methodshort{} enables regularized policy gradient methods to achieve lower exploitability in games with significantly more challenging exploration. We show that augmenting starting state distributions when solving imperfect information games can lead to biased equilibria, and we provide a straightforward mitigation to this in the form of multi-task observation flags. Finally, we release a new set of benchmark environments that drastically increase exploration challenges and state counts in existing OpenSpiel games while keeping exploitability measurements analytically tractable.
\end{abstract}

\input{sections/intro}

\input{sections/related_work}
\input{sections/background}
\input{sections/method}

\input{sections/control_games}
\input{sections/nate_experiments}

\appendix

\subsubsection*{Acknowledgments}
\label{sec:ack}
Author Lanier was supported by a Hasso Plattner Foundation Fellowship. This work was also funded in part by NSF Award 232178 and BSF Grant 2024079.


\bibliography{citations.bib}
\bibliographystyle{rlj}

\beginSupplementaryMaterials

\input{sections/appendix}


\end{document}

%% file: sections/intro.tex
\section{Introduction}

Deep reinforcement learning (RL) has achieved remarkable success in large-scale imperfect-information competitive games, surpassing human professionals in StarCraft \citep{Vinyals2019StarCraft}, Dota \citep{Berner2019Dota2W}, and Poker \citep{brown2020poker}. However, these successes require massive computational resources and time, rendering large-scale game solving impractical for most practitioners. One key factor underlying this high computational cost is the difficulty of exploration in environments with sparse rewards, long decision horizons, and complex hidden-state dynamics. Especially in video games, traditional methods may spend enormous amounts of computation repeatedly exploring long sequences of coordinated intricate actions, even though strategically interesting decisions may often occur only at specific junctures.

A widely-adopted approach to mitigate exploration
difficulty is initializing RL policies with behavioral cloning (BC) from
human demonstrations \citep{silver2016mastering, Vinyals2019StarCraft}.
While effective for bootstrapping basic competence, BC depends on having
enough state-action demonstrations to learn the necessary skills from
which self-play can extrapolate. An alternative use of the same offline data is to reset
training episodes to intermediate states visited in demonstrations,
leveraging the data for \textit{state coverage} rather than
\textit{policy initialization}. This approach needs only a set of visited states rather than full state-action trajectories, which may be available in a broader class of domains. Furthermore, it requires only that the dataset
contains states from which self-play can discover important skills, making it more likely to be viable when collected demonstrations are too few or in an incomplete format for effective behavioral cloning.

Single-agent RL has made substantial gains in sample efficiency by leveraging offline demonstrations more actively through intermediate-state resets, beginning online episodes from informative, intermediate states recorded in demonstrations rather than always restarting from initial conditions \citep{tao2024reverseforwardcurriculumlearning}. Motivated by this success, we propose extending intermediate-state resets to multi-agent competitive settings for the purpose of accelerating equilibrium convergence in large-scale, imperfect-information games.

We focus on games that intuitively decompose into a small number of recurring high-level strategic choices separated by long, highly coordinated control segments. In such games (e.g. real-time strategy games like StarCraft and first-person shooters like CounterStrike), the horizon and state spaces are enormous, but the effective set of distinct coarse strategic options that players employ is comparatively small. This makes it plausible to collect human gameplay that reaches a broad range of strategically relevant subgames, providing a rich set of starting states from which self-play can efficiently discover equilibrium strategies.

\begin{figure}[t!]
    \centering
    \includegraphics[width=\linewidth]{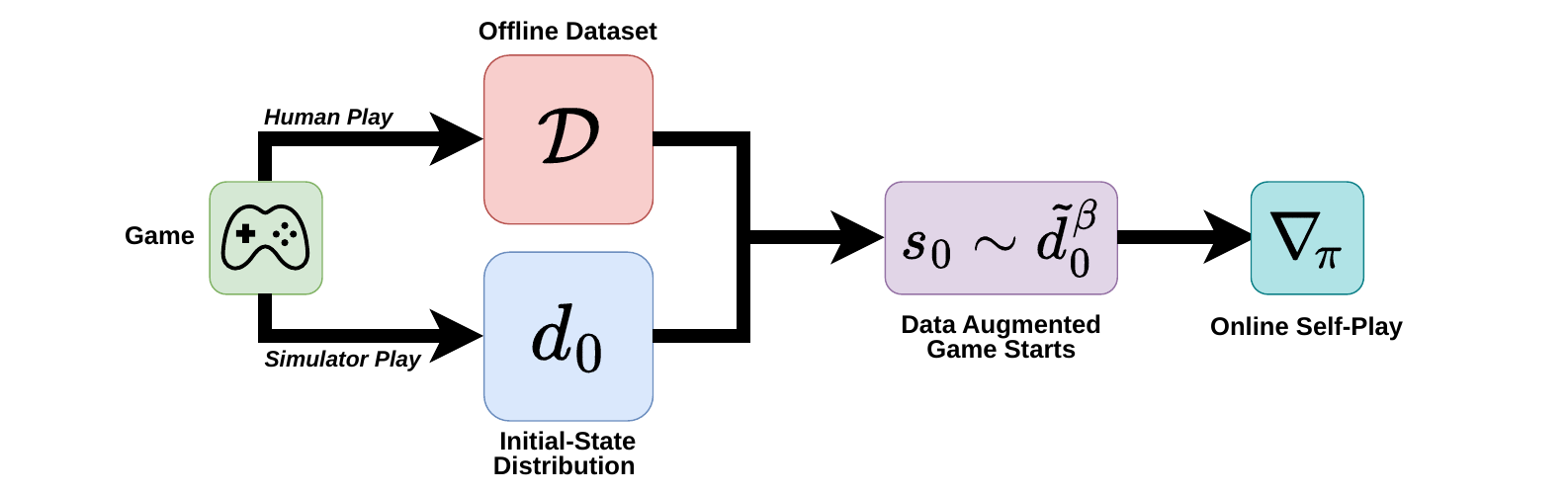}
    \caption{Overview of DAGS. Given a two-player zero-sum game with initial-state distribution $d_0$ and an offline dataset of states $\mathcal{D}$ collected from proficient gameplay, DAGS produces an augmented initial-state distribution $\tilde{d}^\beta_0 = (1-\beta)\,d_0 + \beta\,d_{\mathcal{D}}$ that includes intermediate states from $\mathcal{D}$. Episodes sampled from $\tilde{d}^\beta_0$ are used for self-play training with regularized policy gradient methods. By starting episodes from strategically relevant intermediate states, DAGS reduces the cost of coordinated exploration and enables agents to converge to low-exploitability play in games that are otherwise too large to solve within a limited compute budget.}
    \label{fig:DAGS}
\end{figure}

Our proposed method, \methodlong{} (\methodshort{}), incorporates multi-agent intermediate-state resets into online RL training. \methodshort{} randomly initializes episodes from global states sampled along offline trajectories generated by proficient play that doesn't need to be strategically optimal (such as typical human gameplay). Our main data requirement is that these offline states provide broad coverage of strategically distinct game situations. By resetting into such states, training can concentrate exploration on resolving high-level strategic interactions rather than learning or relearning the long coordinated action sequences needed to reach them from the initial state, leading to faster convergence toward low-exploitability play.

Evaluating equilibrium convergence in large-scale games is difficult due to the computational infeasibility of exact exploitability calculations. To measure learning performance in larger games, we also introduce a set of analytically tractable benchmarks that extend standard OpenSpiel two-player zero-sum games, such as Kuhn Poker and Goofspiel, with complex gridworld control subtasks. In these benchmarks, agents must navigate to specific locations to execute their intended inner-game actions, significantly amplifying the exploration challenge. Crucially, policies learned in these gridworld-extended benchmarks can be cheaply analytically reduced to equivalent policies in their simpler original inner-game counterparts, thus enabling exact exploitability calculations despite large state spaces and episode lengths.

Using these novel benchmarks and synthetic offline data, we empirically demonstrate that \methodshort{} significantly extends the exploration scale of imperfect information games that can be solved within fixed computational budgets using regularized Proximal Policy Optimization (PPO) self-play \citep{rudolph2025reevaluatingpolicygradientmethods}.

However, modifying the starting-state distribution in imperfect-information games can, in principle, distort agents’ beliefs over private information and furthermore change a game's equilibria. We provide a counterexample game illustrating this risk, as well as a straightforward mitigation to the issue. By employing multi-task learning and jointly solving the original game alongside the easier DAGS version of the game, with an in-observation flag for the identity of the game, this equilibrium bias issue can be significantly reduced.

In summary, our contributions are as follows:
\begin{itemize}
    \item We provide a new set of benchmark games for two-player zero-sum settings with long horizons and exploration challenges, designed so that exploitability can be computed analytically.\footnote{Code and environments will be released in a future revision.}
    \item We show that \methodshort{} enables regularized policy gradient algorithms to solve games with significantly larger exploration challenges under a fixed computational budget.
    \item We demonstrate that training with \methodshort{} can bias the learned equilibrium in imperfect-information games, and we provide a mitigation to this bias problem.
\end{itemize}

%% file: sections/related_work.tex
\section{Related Work}
\label{sec:related_work}

\subsection{Competitive Game Solving}

Deep reinforcement learning has achieved landmark results in large-scale competitive games, including Go \citep{silver2016mastering}, StarCraft \citep{Vinyals2019StarCraft}, Dota \citep{Berner2019Dota2W}, Poker \citep{brown2018libratus}, and Stratego \citep{perolat2022mastering, sokota2025stratego}. However, these successes generally have come at the cost of massive computational resources. We focus on reducing the computational cost of solving games that combine long-horizon coordinated control with strategic decision-making under partial observability, a structure common in first-person shooters (e.g., CounterStrike) and real-time strategy games (e.g., StarCraft). In such games, state spaces are enormous and episodes are long, yet the space of high-level strategic decisions is comparatively small. Much of the exploration difficulty lies in reaching strategically relevant states rather than in resolving the strategic interactions themselves. We aim to reduce this cost in model-free self-play by leveraging gameplay data from proficient players to initialize exploration from the states that they naturally visit.

Core approaches to two-player zero-sum imperfect-information game solving include population-based methods such as fictitious play \citep{brown:fp1951}, NFSP \citep{heinrich2016deepreinforcementlearningselfplay}, PSRO \citep{lanctot2017psro, mcaleer2020pipeline, mcaleer2021xdo, mcaleer2024sppsro}, NeuPL \citep{liu2022neupl}; regret-minimization methods such as DREAM \citep{steinberger2020dream} and ESCHER \citep{mcaleerescher}; and regularized policy-gradient methods including NeuRD \citep{omidshafiei2020neurd}, R-NaD \citep{perolat2022mastering}, and magnetic mirror descent \citep{sokota2023a, rudolph2025reevaluatingpolicygradientmethods}. While our approach, augmenting the starting-state distribution, is in principle compatible with any of these families, we focus on its interaction with regularized policy gradients.

In particular, \citet{sokota2023a} show that policy gradients with strong entropy regularization toward a uniform ``magnet'' policy converge to quantal-response equilibria, and that this can be implemented as a modification to PPO self-play \citep{rudolph2025reevaluatingpolicygradientmethods}. Concurrent work by \citet{maidment2026emagnet} extends the moving-magnet concept from \citet{sokota2023a} to deep RL and PPO by regularizing the current policy via KL divergence toward a parameter-space exponential moving average of itself, rather than toward the uniform distribution. This yields lower exploitability across several benchmark games compared to uniform regularization. We use both uniform-magnet regularization \citep{rudolph2025reevaluatingpolicygradientmethods} and EMA-magnet (EMAg) regularization \citep{maidment2026emagnet} as our baseline PPO-based two-player zero-sum game solvers on which we investigate DAGS.

\subsection{State resets and curriculum learning for exploration}
Manipulating starting-state distributions is a proven technique for improving exploration in single-agent reinforcement learning \citep{narvekar2020curriculumlearningreinforcementlearning}. DAGS extends a simple variant of intermediate-state resets, sampling human-visited states from offline data as episode starts, to two-player zero-sum imperfect-information games. In this setting, modifying the start distribution introduces a unique challenge: it can bias the learned equilibrium. We offer a straightforward mitigation via multi-task learning.

In single-agent RL, Reverse Curriculum Generation \citep{florensa2017reverse} generates starts near goal states and progressively expands to harder starting regions, requiring no demonstrations. \citet{salimans2018learning} show that resetting episodes to states drawn from a single demonstration reduces exploration complexity from exponential to quadratic, an approach whose core mechanism, resetting to offline trajectory states, is shared by DAGS. Backplay \citep{resnick2018backplay} schedules demonstration-based resets backward over training. Go-Explore \citep{ecoffet2021goexplorenewapproachhardexploration} maintains an archive of discovered states and resets to them for continued exploration in hard-exploration Atari domains. More recently, JSRL \citep{Ikechukwu2023jumpstart} uses a guide policy to create an explicit curriculum of starting states, and RFCL \citep{tao2024reverseforwardcurriculumlearning} combines reverse and forward scheduling along demonstrations for extreme sample efficiency.

Our approach is most closely related to \citet{salimans2018learning}, which uses a core mechanism identical to ours, resetting episodes to states drawn from offline trajectory data, in the single-agent setting. DAGS extends this to two-player zero-sum self-play with a deliberately simple curriculum design (uniform sampling with a mixing parameter $\beta$). A primary focus and novel contribution of this work is in demonstrating and mitigating the belief bias that start-state augmentation can introduce in imperfect-information games. The only data requirement for DAGS is offline gameplay states that reach hard-to-discover regions of the game tree, rather than near-strategically-optimal demonstrations.

%% file: sections/background.tex
\section{Preliminaries}
\label{sec:preliminaries}

We model our environments as two-player zero-sum partially observable stochastic games (2p0s POSGs) with finite horizon. A 2p0s POSG is a tuple
\[
  \mathcal{G}
  =
  (N, S, A, O, T, \Omega, r, \gamma, d_0),
\]
where $N = \{1,2\}$ is the set of players, $S$ is the state space, $A = A_1 \times A_2$ is the joint action space, $O = O_1 \times O_2$ is the joint observation space, $T$ is the transition kernel, $\Omega$ is the observation function, $r = (r_1,r_2)$ are the per-player rewards with $r_1 = -r_2$, $\gamma \in [0,1)$ is the discount factor, and $d_0$ is the initial-state distribution. Each player $i$ follows a stochastic policy $\pi_i(a \mid o)$ mapping observations to action distributions, and a joint policy $\pi = (\pi_1,\pi_2)$ together with $d_0$ and $T$ induces a distribution over trajectories and returns. We evaluate learned policies using \emph{exploitability}, defined as the average value gain a worst-case best-response opponent can obtain against the current joint policy. Exploitability is zero if and only if the joint policy is a Nash equilibrium.

We build on two regularized PPO self-play solvers. PPO-Uniform \citep{rudolph2025reevaluatingpolicygradientmethods} regularizes toward a uniform ``magnet'' policy, and PPO-EMAg \citep{maidment2026emagnet} regularizes toward an exponential moving average of the policy parameters. We represent the policy with a neural network $\pi_\theta$ and use a shared parameterization for both players in symmetric self-play, with the observation encoding player identity.

%% file: sections/method.tex
\section{Data-Augmented Game Starts (DAGS)}
\label{sec:dags}

DAGS augments self-play training with intermediate-state resets drawn from offline trajectories. We first define the offline state dataset and the augmented initial-state distribution, then describe how DAGS integrates into regularized PPO self-play.

\subsection{Offline state dataset}

We assume access to an offline dataset of global states
\(
  \mathcal{D} = \{ s_m \}_{m=1}^M
\)
collected from trajectories generated by a behavioral policy
\(
  \hat{\pi} = (\hat{\pi}_1, \hat{\pi}_2)
\).
Formally, these states are sampled from the state visitation distribution of $\hat{\pi}$ in the underlying POSG,
\(
  s_m \sim d_{\hat{\pi}},
\)
where $d_{\hat{\pi}}$ denotes the state visitation distribution induced by $(d_0, \hat{\pi})$. In practice, $\hat{\pi}$ is chosen to be competent in the control aspects of the game (e.g., executing basic game mechanics like movement and aiming), but not necessarily strategically near equilibrium. This reflects realistic settings where human or scripted gameplay is available that exhibits high control skill but suboptimal strategy.

Each entry $s_m$ in $\mathcal{D}$ is a full environment state in the POSG, including all hidden information and simulator variables. When we reset to $s_m$ during training, each player $i$ observes only their own observation $o^i = \Omega(s_m, i)$, and the environment continues to evolve according to the original transition dynamics $T$. Thus DAGS modifies only the starting-state distribution while preserving the game’s information structure and dynamics.

\subsection{Augmented initial-state distribution}

DAGS defines an augmented initial-state distribution $\tilde{d}_0$ that places probability mass on both the root states and offline states. We write
\begin{equation}
  \tilde{d}_0^\beta
  =
  (1 - \beta)\, d_0
  +
  \beta \, d_{\mathcal{D}},
  \qquad
  \beta \in [0,1],
  \label{eq:augmented-init}
\end{equation}
where $d_{\mathcal{D}}$ is a uniform distribution over $\mathcal{D}$,
\(
  d_{\mathcal{D}}(s)
  =
  \frac{1}{M} \sum_{m=1}^M \mathbf{1}\{s = s_m\}.
\)
The mixture coefficient $\beta$ controls the fraction of episodes that begin from offline states. The special case $\beta = 0$ recovers standard root-start self-play, while $\beta = 1$ corresponds to always starting from offline states. We leave more sophisticated choices of $d_{\mathcal{D}}$ to future work, focusing here on analyzing the effects of state resets in this basic form and the equilibrium bias they introduce in imperfect-information games.

Training with DAGS can be viewed as solving an \emph{augmented game} in which the only change is that episodes start from $\tilde{d}_0^\beta$ instead of $d_0$. From a reinforcement-learning perspective, this directs exploration toward strategically relevant subgames discoverable from dataset states that are hard to reach by naive exploration from the root. From a game-theoretic perspective, however, replacing $d_0$ with $\tilde{d}_0^\beta$ in an imperfect-information game can alter beliefs over hidden information and shift the game’s equilibrium (Appendix~\ref{app:bias_analysis}).

To mitigate this equilibrium bias, we leverage multi-task learning with a task identification flag in the observation. The augmented starting distribution produces both root-start episodes (sampled from $d_0$) and dataset-start episodes (sampled from $d_{\mathcal{D}}$). We append a binary flag ($f=0$ for root-start, $f=1$ for dataset-start) to each player’s observation indicating which source the episode was drawn from. This conditions the policy so that dataset-start episodes benefit from DAGS-augmented exploration while root-start episodes learn an unbiased policy for the original game (Appendix~\ref{app:bias_analysis}), benefiting from transfer learning through shared network parameters. At evaluation time, we use only the $f=0$ task to recover play under the original starting-state distribution. In the special case $\beta = 1$, where no root-start episodes are generated during training, we evaluate with the flag set to dataset-start.

\subsection{Self-play training with DAGS}
\label{sec:dags-training}

We summarize the complete DAGS training procedure in Algorithm~\ref{alg:dags}. Collected trajectories are used for regularized PPO self-play updates (see Appendix~\ref{app:emag} for details).

\begin{algorithm}[t]
  \caption{Data-Augmented Game Starts (DAGS) with PPO self-play}
  \label{alg:dags}
  \begin{algorithmic}[1]
    \Require POSG $\mathcal{G}$ with initial distribution $d_0$, offline dataset $\mathcal{D}$,
             mixture parameter $\beta \in [0,1]$, joint policy $\pi_\theta = (\pi_{\theta_1}, \pi_{\theta_2})$, PPO hyperparameters
    \While{not converged}
      \State Initialize empty trajectory buffer $\mathcal{T}$
      \For{episode $= 1,\dots,K$}
        \State With probability $\beta$:
        \State \quad Sample initial state $s_0$ from $\mathcal{D}$; set $f \leftarrow 1$ \Comment{DAGS reset}
        \State Otherwise:
        \State \quad Sample initial state $s_0 \sim d_0$; set $f \leftarrow 0$ \Comment{root start}
        \State Roll out self-play episode from $s_0$ using joint policy $\pi_\theta(\cdot \mid o, f)$
        \State Add the resulting trajectory data to $\mathcal{T}$
      \EndFor
      \State $\theta \leftarrow \textsc{PPO Updates}(\theta, \mathcal{T})$
    \EndWhile
  \end{algorithmic}
\end{algorithm}

DAGS offloads the burden of exploring long coordinated control sequences into the offline data collection phase, allowing online self-play to start directly from states where hard-to-find strategically important transitions can be more easily discovered. In the following sections, we introduce analytically tractable benchmarks to evaluate this approach (Section~\ref{sec:benchmarks}) and show empirically that DAGS leads to substantially lower exploitability under fixed computational budgets.

%% file: sections/control_games.tex
\section{Benchmarking Exploration in Large Games}
\label{sec:benchmarks}

To study methods like DAGS in a controlled but challenging setting, we construct a family of benchmark games by applying two successive transformations to analytically tractable two-player zero-sum games from OpenSpiel \citep{lanctot2020openspielframeworkreinforcementlearning}. In this paper, we apply these transformations to Kuhn Poker and Goofspiel, as well as a custom counterexample game designed to expose equilibrium bias from start-state augmentation. The first transformation is to add a forfeit action to every decision node. Then, we replace each decision with a multi-step gridworld navigation task whose outcome determines the base-game action taken. The resulting games have long horizons and large state spaces but remain analytically reducible to their base game, enabling exact exploitability measurement.

\subsection{Forfeit (FF) Transformation}
\label{sec:forfeit}

Given a base two-player zero-sum game with utilities bounded in $[u_{\min}, u_{\max}]$, we augment every decision state with an additional forfeit action. If a player forfeits, the episode terminates immediately. The forfeiting player receives $u_{\min} - 1$ and the opponent receives $-(u_{\min} - 1)$. If both players forfeit simultaneously, player 1 is treated as the forfeiting player. This makes forfeiting strictly worse than any base-game outcome, modeling the consequence of failing to execute a required control task as a strictly dominated strategy.

\subsection{Control Game Transformation}
\label{sec:control_transform}

Given a base game augmented with forfeit actions, we replace each decision with a gridworld navigation problem. At each decision point, the acting player is placed at a fixed starting position on a $G \times G$ grid containing one designated action square for each available base-game action, with all action squares placed equidistant from the start. The player observes their grid position, the time remaining, and the base-game information state, and navigates for a fixed number of steps. Their final position determines the base-game action taken; if they are not on an action square when time expires, the forfeit action is selected instead.

Crucially, agents act independently during navigation and do not interact with or observe the opponent while selecting their action. Control-game policies are therefore analytically reducible to base (FF) game policies, i.e., the base game with added forfeit actions, enabling exact exploitability calculation even in games with significantly larger horizons.

To compute exploitability, we evaluate the control policy over all grid states and timer values for each base-game information state and compute the induced probability of terminating on each action square or forfeiting. These probabilities define an equivalent mixed strategy over base-game actions plus forfeit, from which we compute exact exploitability in the base game. The cost of this reduction is proportional to the number of information states, grid cells, and timer steps, and is orders of magnitude smaller than the cost of self-play training in the configurations we test. In the following section, we instantiate these benchmarks at varying grid sizes to strain exploration and evaluate DAGS.

%% file: sections/nate_experiments.tex
\section{Experiments}

We evaluate whether DAGS accelerates equilibrium convergence as exploration difficulty increases, and whether the equilibrium bias identified in Section~\ref{sec:dags} manifests empirically. Our testbed consists of two established benchmark games, Kuhn poker and 4-card Goofspiel, together with a counterexample game that we constructed to create a clear instance of equilibrium bias and to demonstrate that our proposed mitigation is effective (Figure~\ref{fig:counter_example}).

Each game is tested in three variations. The \textbf{Base (FF)} variant augments the original game only with a forfeit action. It has no control difficulty and is strategically equivalent to the control variants. It serves as an empirical lower bound on achievable exploitability. In the \textbf{Control Travel~10} and \textbf{Control Travel~20} variants, each non-forfeit base-game action requires the agent to first navigate a gridworld to the action's location, with a time limit equal to the required travel distance (10 and 20 respectively) that leaves zero room for error. These variants preserve the strategic structure of the base games while adding substantial exploration difficulty.

For each environment, we generate an offline dataset of 1000 games using a hard-coded policy that selects base-game actions uniformly at random and executes the correct control navigation. Because the behavioural policy is suboptimal, the dataset provides state coverage but does not constitute a low-exploitability policy.

We compare two self-play methods:
\textbf{PPO-Uniform}, which uses entropy regularization toward a uniform magnet~\citep{rudolph2025reevaluatingpolicygradientmethods}, and
\textbf{PPO-EMAg}, which uses KL regularization toward an exponential moving average of policy parameters~\citep{maidment2026emagnet}.
Base PPO hyperparameters follow \citet{rudolph2025reevaluatingpolicygradientmethods}.
Throughout, $\beta=0$ corresponds to standard self-play from the root (no DAGS) and serves as our baseline.
For each combination of method, environment, variation, and $\beta$, we sweep over regularization hyperparameters and select the best configuration. Full sweep details are provided in Appendix~\ref{app:hyperparams}.

\subsection{DAGS Improves Scaling in Kuhn and Goofspiel}

\begin{figure}[h!]
    \centering
    \includegraphics[width=\linewidth]{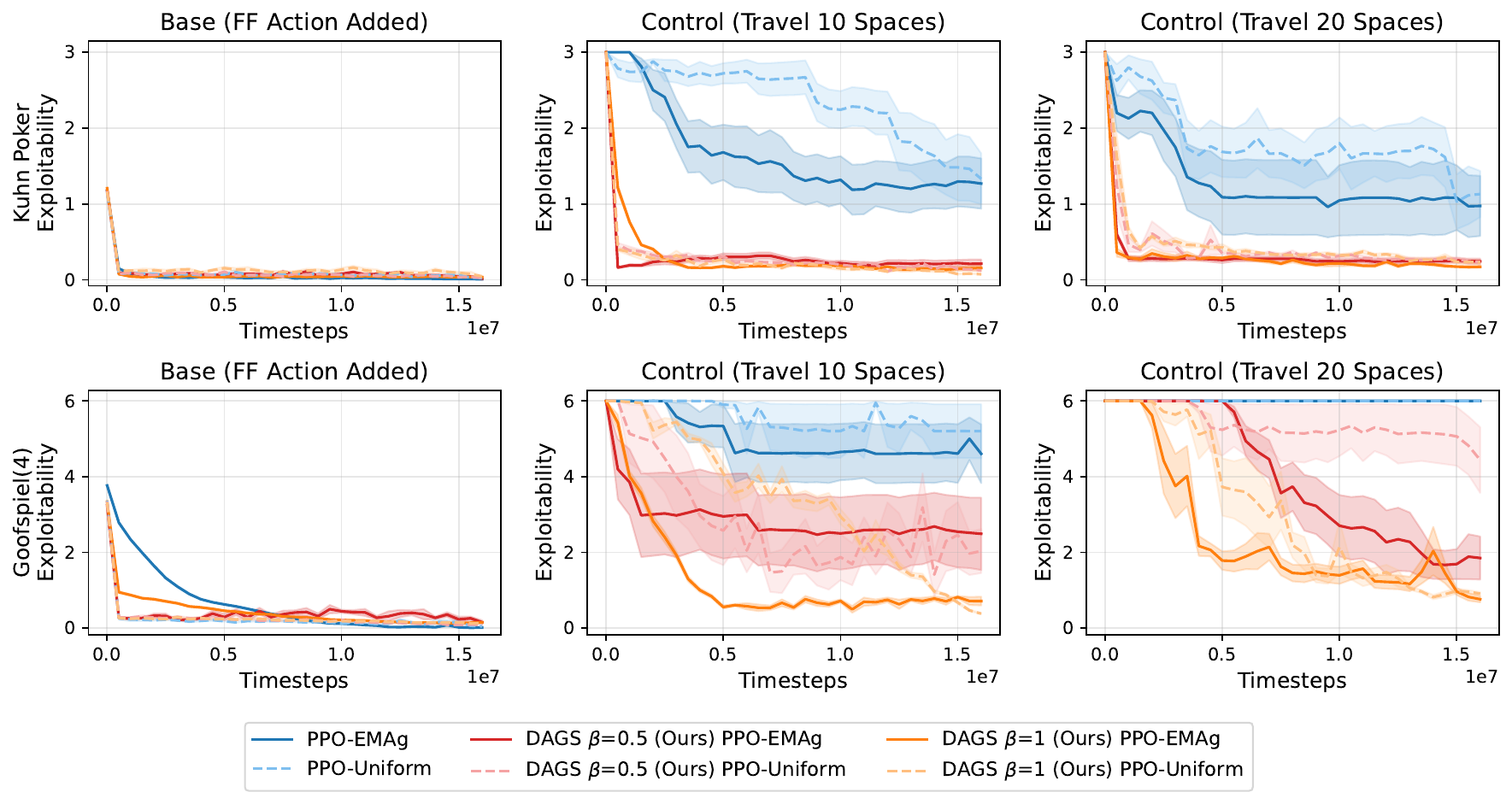}
    \caption{Exploitability over training for Kuhn poker and 4-card Goofspiel across game variations and $\beta$ values, using both PPO-Uniform and PPO-EMAg. DAGS ($\beta > 0$) increasingly outperforms standard self-play ($\beta = 0$) as exploration difficulty grows.}
    \label{fig:scaling}
\end{figure}

We first investigate how DAGS performs on Kuhn Poker and Goofspiel as exploration difficulty increases, starting with the base (FF) version of each game, then increasing difficulty in the Control Travel~10 variant and again in Control Travel~20 (Figure~\ref{fig:scaling}).

In the base (FF) variants, all methods and $\beta$ values converge to similarly low exploitability, confirming that exploration is not a bottleneck in these games. However, as exploration complexity increases in the gridworld Control variants, both games show the same pattern: the gap between DAGS and the $\beta=0$ baseline widens, as standard self-play fails to discover how to navigate to certain action squares within the time limit, effectively forfeiting at those decision nodes. DAGS, by contrast, initializes episodes at states where agents can immediately discover how to perform these navigation tasks, maintaining lower exploitability. The effect is especially pronounced in Goofspiel. At Travel~10, the Goofspiel standard self-play stalls at an exploitability of ${\sim}4$-$6$, while DAGS variants converge to substantially lower values. At Travel~20, standard self-play fails entirely; Goofspiel with $\beta=0$ plateaus at ${\sim}6.0$, representing no meaningful learning, while DAGS continues to converge in both games. In Goofspiel, $\beta=1$ outperforms $\beta=0.5$, likely because root-start episodes in high-difficulty control variants are largely uninformative (the agent forfeits), so $\beta=0.5$ effectively wastes half its training compute on episodes that provide a less useful learning signal.

Interestingly, despite the theoretical possibility of equilibrium bias (Section~\ref{sec:dags}), we see little evidence of this in either game. All $\beta$ values converge to similar exploitability in the base variants, and since control-game policies can be reduced to base-game policies, this absence of bias carries over to the control variants as well. To understand when bias does arise, we turn to a purpose-built counterexample designed to elicit it.

\begin{figure}[h]
    \centering
    \includegraphics[width=0.7\linewidth]{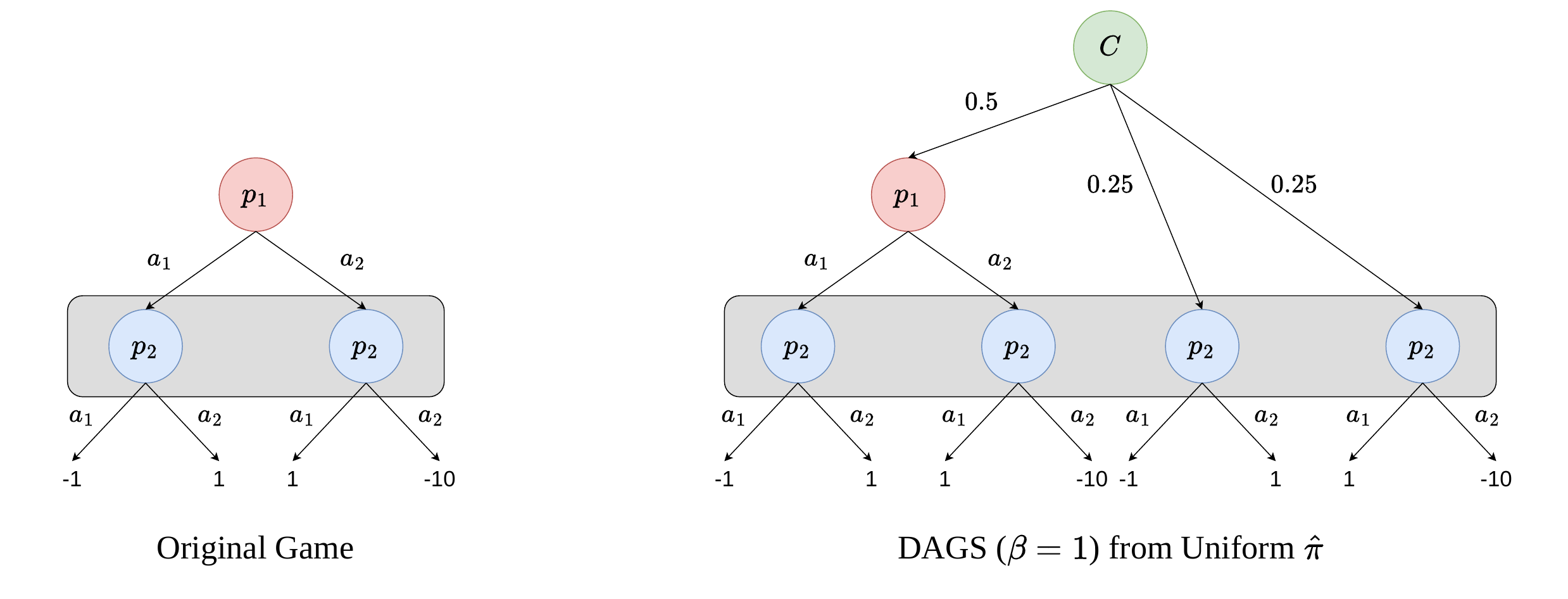}
    \caption{The counterexample game, a two-player zero-sum game whose NE changes when DAGS with $\beta = 1$ augments the initial state distribution using states from a strategically uniform behavioral policy (Appendix~\ref{app:environments}). The grey box is a single information state that appears identical to player 2. In the original game (left), $p_2$'s posterior over $p_1$'s action depends entirely on $p_1$'s strategy. With $\beta=1$ (right), the chance node resets directly to $p_2$'s decision nodes with equal probability regardless of $p_1$'s action, biasing $p_2$'s belief of the game state.}
    \label{fig:counter_example}
\end{figure}

\subsection{Equilibrium Bias and Mitigation in the Counterexample Game}
\label{sec:ce}

\begin{figure}[h]
    \centering
    \includegraphics[width=0.7\linewidth]{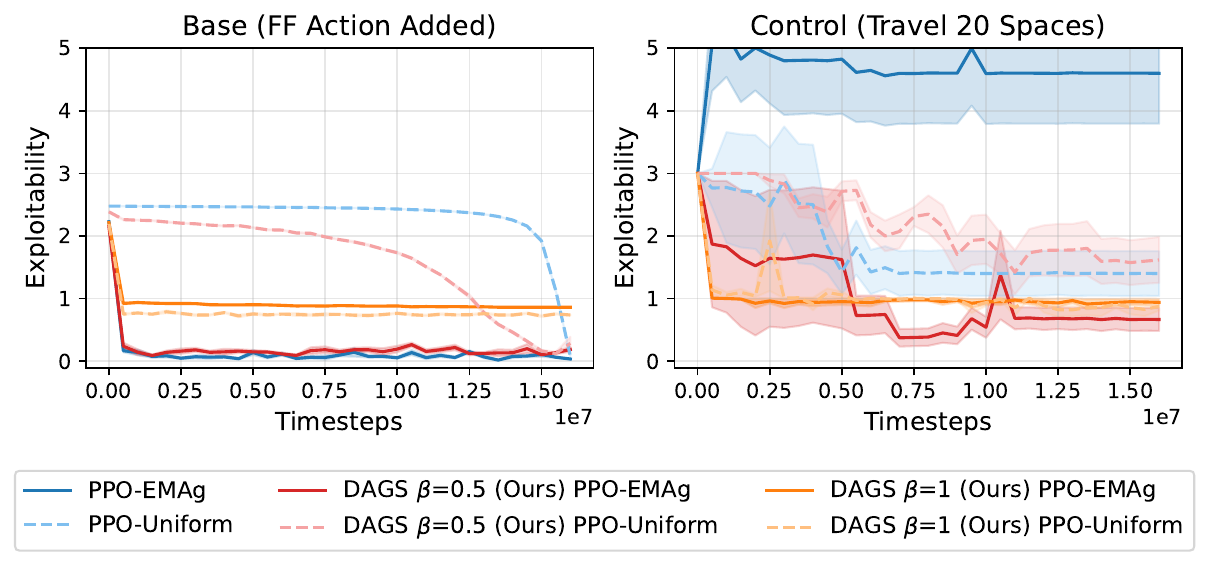}
    \caption{Exploitability over training for the counterexample game across Base (FF) and Control Travel 20 variants. $\beta=1$ induces equilibrium bias (high exploitability) even in the base game. $\beta=0$ avoids bias but fails to explore in the control variant. $\beta=0.5$ with a task-identifying observation flag achieves effective exploration while mitigating the equilibrium bias.}
    \label{fig:ce_results}
\end{figure}

In the counterexample game (Figure~\ref{fig:counter_example}), augmenting the start-state distribution alters beliefs over private information and shifts the equilibrium. Figure~\ref{fig:ce_results} shows results for the Base (FF) and Control Travel~20 variants.
In the base game, standard self-play ($\beta=0$) converges to near-zero exploitability with both methods, and $\beta=0.5$ also converges well. However, $\beta=1$ performs poorly. Exploitability settles above $0.7$, confirming that the augmented start-state distribution distorts beliefs over hidden information and shifts the equilibrium away from the original game's NE. Since the base game provides an empirical lower bound on achievable exploitability, this bias necessarily persists in the control variant. In Control Travel~20, the tradeoff between exploration benefit and equilibrium bias becomes apparent.

Without DAGS, exploitability remains very high. With $\beta=1$, agents explore effectively but converge to the biased equilibrium. Setting $\beta=0.5$ provides a middle ground. A task-identifying observation flag (Section~\ref{sec:dags}) partitions the policy into an $f=1$ task that benefits from DAGS-augmented exploration and an $f=0$ task that learns an unbiased policy under the original starting distribution. The $f=0$ branch benefits from multi-task transfer through shared representations learned on the easier $f=1$ task. At evaluation, setting $f=0$ recovers unbiased play, whereas $\beta=1$ must evaluate with $f=1$ since no root-start episodes were seen during training.

These results demonstrate that when DAGS introduces equilibrium bias, training with intermediate $\beta$ and a task-identifying observation flag mitigates the bias while still leveraging augmented exploration.

\section{Conclusion and Future Work}
We present \methodlong{} (\methodshort{}), a starting-state sampling strategy that initializes self-play episodes at intermediate states drawn from offline gameplay data, leveraging demonstrations for state coverage rather than policy initialization. To evaluate \methodshort{} in a controlled setting, we introduce analytically reducible benchmark games that wrap standard two-player zero-sum games with gridworld control tasks, greatly increasing horizon and exploration difficulty while preserving tractable exploitability computation. Empirically, \methodshort{} outperforms standard self-play in control variants of Kuhn poker, 4-card Goofspiel, and a counterexample game designed to expose equilibrium bias, enabling regularized policy gradient methods to solve games with significantly larger exploration challenges under fixed computational budgets. We also show that augmenting starting-state distributions in imperfect-information games can bias the learned equilibrium, and we provide a straightforward mitigation through multi-task training with observation flags that recovers unbiased play in the counterexample game.

Our gridworld benchmarks use simple mechanics by design to preserve analytical reducibility. In this setting, warm-starting RL with BC may also be effective from limited data since the observation-action mapping is relatively straightforward to clone. In domains with complex, context-dependent mechanics (e.g., fighting games or first-person shooters), we expect DAGS's advantage over BC-based initialization to grow as demonstrations become scarcer, since BC needs substantially more data to learn useful behavior while DAGS requires only that the dataset reach states from which important strategies can be discovered. We see applying DAGS to complex games where exact exploitability is not tractable as promising future work.

%% file: sections/appendix.tex
\section{Hyperparameter and Training Details}
\label{app:hyperparams}

Table~\ref{tab:ppo_hyperparams} lists the base PPO hyperparameters used across all experiments. Table~\ref{tab:sweep} lists the hyperparameters swept for each method.

\begin{table}[h]
\centering
\caption{Base PPO hyperparameters (shared across all methods and environments).}
\label{tab:ppo_hyperparams}

\begin{tabular}{ll}
\toprule
Hyperparameter & Value \\
\midrule
Learning rate & $2.5 \times 10^{-4}$ \\
Clip ratio ($\epsilon$) & $0.1$ \\
PPO epochs per update & $4$ \\
Batch size (num\_envs $\times$ num\_steps) & $8 \times 128 = 1024$ \\
GAE $\lambda$ & $0.95$ \\
Discount $\gamma$ & $1.0$ \\
Value function coefficient & $0.5$ \\
Max gradient norm & $0.5$ \\
Total training steps & $16 \times 10^6$ \\
\bottomrule
\end{tabular}
\end{table}

\begin{table}[h]
\centering
\caption{Hyperparameters swept for each method. For each combination of environment, game variation, and $\beta$, we select the configuration with the lowest final exploitability. PPO-EMAg used Bayesian optimization (200 trials) and PPO-Uniform used a grid search (56 combinations).}
\label{tab:sweep}

\small
\begin{tabular}{llp{6.2cm}}
\toprule
Method & Hyperparameter & Values swept \\
\midrule
PPO-Uniform & Entropy coefficient & \{0, 0.001, 0.01, 0.02, 0.05, 0.1, 0.25, 0.5, 1, 2, 4, 8, 16, 32\} \\
PPO-Uniform & Anneal entropy & \{True, False\} \\
PPO-Uniform & Anneal learning rate & \{True, False\} \\
\midrule
PPO-EMAg & KL coefficient ($\lambda_{\mathrm{KL}}$) & \{0.05, 0.1, 0.25, 0.5, 1, 2, 4, 8, 16, 32\} \\
PPO-EMAg & EMA step size ($\tau$) & \{0.1, 0.05, 0.01, 0.005, 0.001, 0.0005, 0.0001\} \\
PPO-EMAg & Entropy coefficient & \{0, 0.001, 0.01\} \\
PPO-EMAg & Anneal learning rate & \{True, False\} \\
\bottomrule
\end{tabular}
\end{table}

\textbf{Network architecture.}
Following \citet{rudolph2025reevaluatingpolicygradientmethods}, both the actor and critic are fully-connected networks with 3 hidden layers of 512 units and Tanh activations. Weights are initialized with orthogonal initialization (output layer standard deviation of 0.01 for the actor and 1.0 for the critic). A single set of parameters is shared across both players in symmetric self-play, with player identity encoded in the observation.

\section{PPO Regularization Details}
\label{app:emag}

Algorithm~\ref{alg:dags} abstracts the policy update as \textsc{PPOUpdates}$(\theta, \mathcal{T})$. We describe the two regularization schemes used in our experiments.

\textbf{PPO-Uniform} \citep{rudolph2025reevaluatingpolicygradientmethods} augments the standard clipped PPO objective with an entropy bonus that regularizes the policy toward the uniform distribution:
\begin{equation}
  \mathcal{L}_{\text{PPO-Uniform}}(\theta)
  = \mathcal{L}_{\text{PPO}}(\theta)
  - \lambda_H \, \mathbb{E}_{s \sim \mathcal{T}}\!\bigl[H\bigl(\pi_\theta(\cdot \mid s)\bigr)\bigr],
\end{equation}
where $\mathcal{L}_{\text{PPO}}(\theta)$ is the standard clipped surrogate objective \citep{schulman2017proximalpolicyoptimizationalgorithms}, $H(\cdot)$ denotes the entropy, and $\lambda_H > 0$ controls the regularization strength. \citet{rudolph2025reevaluatingpolicygradientmethods} show that this performs similarly to magnetic mirror descent toward a uniform magnet.

\textbf{PPO-EMAg} \citep{maidment2026emagnet} replaces the uniform magnet with an exponential moving average (EMA) of the policy parameters. The objective becomes
\begin{equation}
  \mathcal{L}_{\text{PPO-EMAg}}(\theta)
  = \mathcal{L}_{\text{PPO}}(\theta)
  + \lambda_{\mathrm{KL}} \, \mathbb{E}_{s \sim \mathcal{T}}\!\bigl[D_{\mathrm{KL}}\bigl(\pi_{\theta_{\mathrm{mag}}}(\cdot \mid s) \,\|\, \pi_\theta(\cdot \mid s)\bigr)\bigr],
\end{equation}
where $\theta_{\mathrm{mag}}$ denotes the magnet parameters. After each PPO update step, the magnet is updated as
\begin{equation}
  \theta_{\mathrm{mag}} \leftarrow (1 - \tau) \, \theta_{\mathrm{mag}} + \tau \, \theta,
\end{equation}
with step size $\tau \in (0,1]$. This yields lower exploitability than uniform regularization across several benchmark games \citep{maidment2026emagnet}.

\section{Equilibrium Bias Analysis}
\label{app:bias_analysis}

We formalize the observation that augmenting the start-state distribution can shift a game's Nash equilibrium, and show that conditioning on an episode-source flag with intermediate $\beta$ recovers the original equilibrium.

\begin{proposition}[Start-state augmentation can shift the NE]
\label{prop:bias}
There exist two-player zero-sum POSGs and behavioral policies $\hat{\pi}$ such that, for any $\beta > 0$, the set of Nash equilibria of the game played under $\tilde{d}_0^\beta = (1-\beta)\,d_0 + \beta\,d_{\mathcal{D}}$ differs from the set of Nash equilibria under $d_0$.
\end{proposition}

\begin{proof}
We give a constructive proof using the counterexample game in Figure~\ref{fig:counter_example}. Since $p_2$ cannot observe $p_1$'s action, the game is strategically equivalent to a simultaneous-move matrix game with $p_1$'s payoff matrix
\[
A = \begin{pmatrix} -1 & 1 \\ 1 & -10 \end{pmatrix},
\]
where rows correspond to $p_1$'s actions $(a_1, a_2)$ and columns to $p_2$'s actions $(a_1, a_2)$. Let $p$ denote $p_1$'s probability of playing $a_1$ and $q$ denote $p_2$'s probability of playing $a_1$, with mixed strategies $\mathbf{p} = (p,\; 1{-}p)^\top$ and $\mathbf{q} = (q,\; 1{-}q)^\top$. In the original game, $p_1$ maximizes and $p_2$ minimizes $\mathbf{p}^\top A\, \mathbf{q}$.

\textbf{Original game NE.}\quad
The unique NE is $p = q = \tfrac{11}{13}$, with game value $-\tfrac{9}{13}$ for $p_1$. Both players receive equal expected payoff from either action, and no pure-strategy NE exists.

\textbf{DAGS breaks this equilibrium.}\quad
Let $\hat{\pi}$ be uniform over actions. Since every trajectory visits the root and one of $p_2$'s nodes, $d_{\mathcal{D}}$ places mass $\tfrac{1}{2}$ on the root, $\tfrac{1}{4}$ on ``$p_2$ after $a_1$'', and $\tfrac{1}{4}$ on ``$p_2$ after $a_2$''. Under $\tilde{d}_0^\beta$ with $\beta > 0$, a fraction of episodes reset directly to $p_2$'s nodes with equal probability, bypassing $p_1$. Since $p_2$ cannot distinguish root-start from dataset-start episodes, $p_2$ must use the same strategy in both. In dataset-start episodes, $p_2$ faces each of $p_1$'s actions with equal probability regardless of $p_1$'s strategy. Under this uniform distribution ($p = \tfrac{1}{2}$), $p_1$'s expected payoff from $p_2$'s columns is $a_1\!: 0$ versus $a_2\!: -\tfrac{9}{2}$, so $p_2$ strictly prefers $a_2$. At the original NE, $p_2$ is exactly indifferent between columns. Mixing in any $\beta > 0$ fraction of dataset-start episodes breaks this indifference in favor of $a_2$, giving $p_2$ a profitable deviation. Therefore, the original NE is not a NE of the augmented game for any $\beta > 0$.
\end{proof}

\begin{proposition}[Flag conditioning recovers the original NE]
\label{prop:flag}
Let $\beta \in (0,1)$ and let each player's policy be conditioned on the binary flag $f \in \{0,1\}$, where $f=0$ for episodes drawn from $d_0$ and $f=1$ for episodes drawn from $d_{\mathcal{D}}$. If $(\pi_1^*, \pi_2^*)$ is a Nash equilibrium of the augmented game with flag-conditioned policies, then the restriction $(\pi_1^*(\cdot \mid \cdot, f\!=\!0),\, \pi_2^*(\cdot \mid \cdot, f\!=\!0))$ is a Nash equilibrium of the original game under $d_0$.
\end{proposition}

\begin{proof}
Since $f$ is a deterministic function of the episode's start-state source and is included in every observation, each player's policy $\pi_i(\cdot \mid o, f)$ decomposes into two independent tasks, one for $f=0$ episodes and one for $f=1$ episodes. The $f=0$ episodes are drawn from $d_0$ and evolve under the original transition dynamics, so they constitute an instance of the original game. The $f=1$ episodes are drawn from $d_{\mathcal{D}}$ and constitute a separate game under the augmented distribution. Because the flag partitions each player's policy into disjoint parameters ($\pi_i(\cdot \mid \cdot, f\!=\!0)$ and $\pi_i(\cdot \mid \cdot, f\!=\!1)$) that affect disjoint terms in each player's expected return, a unilateral deviation in one sub-policy cannot affect the other sub-game's payoff. In particular, the $f=0$ task is optimized over episodes drawn from $d_0$ alone, and any NE of this sub-problem is a NE of the original game. Therefore, the $f=0$ restriction of any NE of the flag-conditioned augmented game is a NE of the original game.
\end{proof}

Proposition~\ref{prop:flag} guarantees that the $f\!=\!0$ policy is unbiased in the game-theoretic sense. In practice, the $f\!=\!1$ task may additionally benefit learning by providing a richer distribution of training states, enabling multi-task transfer of learned representations to the harder $f\!=\!0$ task. Our experiments in Section~\ref{sec:ce} confirm this. Setting $\beta=0.5$ with flag conditioning recovers low exploitability in the counterexample game while still benefiting from DAGS-augmented exploration in the control variants.

\section{Environment and Benchmark Details}
\label{app:environments}

\textbf{Base games.}
We use two base games from OpenSpiel \citep{lanctot2020openspielframeworkreinforcementlearning}: Kuhn poker (3 cards, 2 players) and 4-card Goofspiel (4 cards per suit, 2 players). Both are small enough for exact exploitability computation yet strategically nontrivial.

\textbf{Forfeit transformation.}
As described in Section~\ref{sec:forfeit}, we augment every decision node with a forfeit action. The forfeit payoff is $u_{\min} - 1$ for the forfeiting player and $-(u_{\min} - 1)$ for the opponent, ensuring forfeiting is strictly dominated by any base-game outcome.

\textbf{Control game transformation.}
At each base-game decision point, the acting player is placed at a fixed starting position on a $G \times G$ grid. One designated action square is placed for each available non-forfeit base-game action, all at equal Manhattan distance from the start. The player's observation at each control step consists of:
\begin{itemize}
    \item Current player's grid position (row, column)
    \item Time remaining (steps until deadline)
    \item The base-game information state (the same observation the player would receive in the base game)
    \item The episode-source flag $f \in \{0, 1\}$
\end{itemize}
Players do not observe the opponent's grid position during navigation, which allows control-game policies to be analytically reduced to policies in the base game. The grid is empty aside from the action squares (no obstacles or other features). The action space during navigation is $\{\text{up, down, left, right, stay}\}$. The time limit equals the Manhattan distance from start to the action squares (the ``travel distance''), leaving zero room for wasted steps. If the player does not occupy an action square when time expires, the forfeit action is taken.

\textbf{Game variations.}
\begin{itemize}
    \item \textbf{Base (FF):} Base game with forfeit actions added and no gridworld navigation. Strategically equivalent to the control variants.
    \item \textbf{Control Travel~10:} Grid sized so that the travel distance from start to each non-forfeit action square is 10 steps. Kuhn and counterexample variants use an $11 \times 11$ grid with 2 action squares and Goofspiel uses a $13 \times 13$ grid with 4 action squares.
    \item \textbf{Control Travel~20:} Travel distance of 20 steps. Kuhn and counterexample variants use a $21 \times 21$ grid with 2 action squares and Goofspiel uses a $25 \times 25$ grid with 4 action squares.
\end{itemize}

\textbf{Counterexample game.}
The counterexample game (Figure~\ref{fig:counter_example}) is a two-player zero-sum game in which $p_1$ chooses $a_1$ or $a_2$, after which $p_2$ chooses $a_1$ or $a_2$ without observing $p_1$'s action (both of $p_2$'s decision nodes form a single information set). The equivalent normal-form payoff matrix for $p_1$ is given in the proof of Proposition~\ref{prop:bias}. This game is designed so that augmenting the starting-state distribution with states from a uniform behavioral policy shifts $p_2$'s beliefs over $p_1$'s action, changing the equilibrium.

\textbf{Dataset generation.}
For each environment, we generate an offline dataset of 1000 complete episodes. The hard-coded behavioral policy selects base-game actions uniformly at random and executes optimal (shortest-path) navigation in the control gridworld to reach the corresponding action square. This policy is mechanically competent but strategically suboptimal, providing broad state coverage without constituting a low-exploitability strategy.